\pdfoutput=1

\documentclass[twocolumn]{article}

\usepackage[utf8]{inputenc} % allow utf-8 input
\usepackage[T1]{fontenc}    % use 8-bit T1 fonts
\usepackage{times}
\usepackage{hyperref}       % hyperlinks
\usepackage{url}            % simple URL typesetting
\usepackage{authblk}
\usepackage{placeins}
\usepackage{ragged2e}

\usepackage{nicefrac}       % compact symbols for 1/2, etc.

\usepackage{amsmath}
\usepackage{amsfonts}
\usepackage{amssymb}
\usepackage{amsthm}
\usepackage{float}
\usepackage{parskip}
\usepackage{titlesec}

\usepackage{microtype}
\usepackage{graphicx}
\usepackage{placeins}
\usepackage{subfigure}
\usepackage{booktabs} % for professional tables

\usepackage{xcolor}

\usepackage{array}   % for \newcolumntype macro
\newcolumntype{L}{>{$}l<{$}} % math-mode version of "l" column type

\DeclareMathOperator*{\argmax}{arg\,max}

\usepackage{bbm}
\usepackage{dsfont}
\newcommand{\vertiii}[1]{{\left\vert\kern-0.25ex\left\vert\kern-0.25ex\left\vert #1 
		\right\vert\kern-0.25ex\right\vert\kern-0.25ex\right\vert}}
	
\labelwidth\leftmargini\advance\labelwidth-\labelsep \labelsep 5pt 
\titleformat*{\section}{\Large\bfseries\raggedright}

\numberwithin{equation}{section}

\newtheorem{lemma}{Lemma}
\newtheorem{definition}{Definition}
\newtheorem{theorem}{Theorem}

\numberwithin{lemma}{section}
\numberwithin{theorem}{section}
\numberwithin{definition}{section}
\numberwithin{remark}{section}
\newtheorem{corollary}{Corollary}
\numberwithin{corollary}{section}

	\title{A Lower Bound for the Sample Complexity of Inverse Reinforcement Learning}

\date{}
\author[1]{\textbf{Abi Komanduru}}
\author[1]{\textbf{Jean Honorio}}
\affil[1]{Purdue University, West Lafayette IN 47906}

\begin{document}

\maketitle
\begin{abstract}
	Inverse reinforcement learning (IRL) is the task of finding a reward function that generates a desired optimal policy for a given Markov Decision Process (MDP). This paper develops an information-theoretic lower bound for the sample complexity of the finite state, finite action IRL problem. A geometric construction of $\beta$-strict separable IRL problems using spherical codes is considered. Properties of the ensemble size as well as the Kullback-Leibler divergence between the generated trajectories are derived. The resulting ensemble is then used along with Fano's inequality to derive a sample complexity lower bound of $O(n \log n)$, where $n$ is the number of states in the MDP.
\end{abstract}

\section{Introduction}

\pagenumbering{gobble}
Reinforcement learning (RL) focuses on the well studied problem of finding an optimal policy for a given Markov Decision Process (MDP) with a known reward function. Inverse Reinforcement Learning (IRL) \cite{IRLNg} considers a MDP with known optimal policy and aims to find a reward function that generates the desired optimal policy. It is well known that the choice of such reward function is not necessarily unique. The IRL problem occurs in situations where the actions of an \textit{expert}, which represent the optimal policy, are known or can be observed and are to be replicated through the proper choice of a reward function. Examples include cases such as apprenticeship learning. 

Two major formulations of the IRL problem have been proposed. The first is the standard MDP formulation considered by \cite{IRLNg}. This is the formulation that is considered in this paper and for which the results are derived. The second is the linearly-solvable MDP (LMDP) formulation of \cite{LMDP2010Dvi}. As noted in \cite{LMDP2010Dvi}, while the standard MDP problem can be embedded in LMDP, solutions to standard MDP problems based on standard MDPs are guaranteed to generate the desired Bellman optimal policy given the true transition probabilities whereas the LMDP formulation does not. Both formulations are used successfully in practice. Methods to solve the standard MDP formulations include the methods presented in \cite{IRLNg}, \cite{AbbeelNg}, Multiplicative Weights for Apprenticeship Learning \cite{syed2008apprenticeship}, Bayesian Estimation IRL \cite{RamaBayes}, Maximum Margin Planning \cite{ratliff2006maximum}, Hybrid IRL \cite{neu2007apprenticeship} and the L1 SVM formulation\cite{komanduru2019correctness}. Examples of methods for solving the LMDP formulation are Maximum-Entropy IRL \cite{ziebart2008maximum} and Gaussian Process IRL \cite{levine2011nonlinear}.

As shown in our previous work \cite{komanduru2019correctness}, various solutions to the standard MDP problem can fail to result in a reward that uniquely generates the desired optimal policy. This failure can render such reward function solutions useless in the case where the goal is to replicate the policy of the expert and not just simply achieve similar values for the value function. With this in mind, in  \cite{komanduru2019correctness} we derived an upper bound for the sample complexity of the inverse reinforcement learning problem. In this paper, we derive an information-theoretic lower bound for the sample complexity of the standard MDP IRL problem when the transition probabilities are estimated from observed trajectories. The derived sample complexity is with respect to the goal of recovering a reward function that correctly generates the desired optimal policy.

 We use Fano's inequality \cite{cover2006elements} to prove our result through a careful construction of an ensemble. The use of restricted ensembles is customary for information-theoretic lower bounds \cite{santhanam2012information}, \cite{wang2010information}, \cite{tandon2014information}. To the best of our knowledge, no such information-theoretic sample complexity lower bound exists for the recovery of the reward function with the stated properties in the case of standard MDP IRL.

In the following section, we review the basic notation of the Inverse Reinforcement Learning problem, the conditions for Bellman optimality of the optimal policy for standard MDP and the notion of $\beta$-strict separability. In Section \ref{sec:GeoCon}, we describe the geometric construction of the ensemble using spherical codes. Section \ref{sec:samp_comp} derives bounds for the cardinality and the KL divergence within the ensemble and culminates with the $O(n\log n)$ lower bound for the sample complexity of inverse reinforcement learning. Section \ref{sec:exp} provides results from simulated experiments using various solutions methods to support our sample complexity bound. 

\section{Preliminaries and Notation}
Consider the standard Markov Decision Process  $(S,A,\lbrace P_{a}\rbrace, \gamma, R)$, where
~\\
\begin{itemize}
	
	\item $S$ is a finite set of $n$ states.
	\item $A = \lbrace a_1, \ldots, a_k\rbrace$ is a set of $k$ actions.
	\item $ P_{a} \in \left[0,1\right]^{n\times n} $ are the state transition probabilities for action $a$. We use $ P_{a}(i) \in \left[0,1\right]^{n}$ to represent the state transition probabilities for action $a$ in the $i$-th state or more simply, the $i$-th row of the transition probability matrix $P_a$
	\item $\gamma \in [0,1]$ is the discount factor.
	\item $R: S\to\mathbb{R}$ is the reward function.
\end{itemize}

In our representation, we consider the $P_{a}$ to be right stochastic matrices, i.e.,

$P_a(i,j) \ge 0 \; \forall \; i,j \;\;\; \text{and}\;\;\;\sum_{j} P_a (i,j) = 1 \; \forall i$

where $P_a (i,j)$ are the entries of $P_a (i)$ and represent the transition probability of going from state $i$ to state $j$ when taking action $a$.

We assume the reward function to depend only on the state instead of the state and the action. This assumption is also made for the prior results in \cite{IRLNg}.

Given a standard MDP, a policy is defined as a map $\pi : S\to A$. Given a policy $\pi$, we can define two functions.\\~\\
The first is the \textit{value function} at a state $s_1$ which is defined as 
\[
V^{\pi}(s_1) = \mathbb{E}\bigl[R(s_1)+\gamma R(\tau(s_1)) +\\ \gamma^2 R(\tau(\tau(s_1 ))) + \ldots \mid \pi \bigr]
\]
where $\tau(s)$ represents the trajectory under policy $\pi$. The second is the \textit{$Q$ function} which is defined as
$$
Q^{\pi}(s,a)  =  R(s) + \gamma \mathbb{E}_{s'\sim P_{a}(s)}[V^\pi (s')]
$$
The \textit{Bellman Optimality equation} states that a policy $\pi^*(s)$ is an optimal policy for an MDP if and only if 
$$
\pi^*(s) \in \argmax_{a\in A} Q^{\pi^*}(s,a),\;\;\; s\in S
$$
The \textbf{Inverse Reinforcement Learning problem} for a standard MDP is posed as the problem of finding the reward function $R$ that generates a desired optimal policy $\pi^*$ given a MDP without reward (MDP$\setminus R$) and the optimal policy $\pi^*$ to be generated.

\cite{IRLNg} prove that for a finite-state MDP with reward $R$, and for $\pi^*\equiv a_1$, the Bellman optimality equation is equivalent to the following condition:

\begin{multline}\label{eq:BellOpt}
(P_{a_1}(i) - P_a(i) )(I-\gamma P_{a_1} )^{-1}R\ge0 \;\;\;\\ \forall\; a\in A\setminus a_1, i=1,\ldots,n
\end{multline}

It can also be shown that  $\pi^*\equiv a_1$ is the unique optimal policy if the above inequality is strict. We note that this condition is necessary and sufficient for the policy to be optimal for the reward. 

This condition also forms the basis for the $\beta$-strict separability which we introduce in \cite{komanduru2019correctness} and will make use of in our construction. We reproduce the aforementioned condition here for convenience.
\begin{definition}[\textbf{$\beta$-Strict Separability}]
	Let $\beta >0$. An inverse reinforcement learning problem $(S,A,P_a,\gamma)$ with optimal policy $\pi^*\equiv a_1$ satisfies \textit{$\beta$-strict separability} if and only if there exists a reward function $R^* : S\to \mathbb{R}$ that satisfies Bellman optimality strictly. More formally,

	$$
	\lVert R^*\rVert_1=1
	$$
	and
	\begin{multline*}
	(P_{a_1}(i) - P_a(i) )(I-\gamma P_{a_1} )^{-1} R^* \ge \beta > 0 \;\;\;\\ \forall a\in A\setminus a_1, i = 1,\ldots, n
	\end{multline*}
\end{definition}

There are various formulations and solution methods for the standard MDP Inverse Reinforcement Learning problem such as those presented in \cite{IRLNg}, \cite{RamaBayes}, \cite{syed2008apprenticeship} and \cite{komanduru2019correctness} to list a few. Our concern in this paper is not the particular method of solution. Instead we seek to provide an information-theoretic lower bound for the sample complexity of the IRL problem. To achieve this, we use Fano's inequality \cite{cover2006elements} along with the construction of an ensemble. 

Throughout this paper we will use $\mathcal{F}$ to represent MDP$\setminus R$ problems of the specified construction as described in Section \ref{sec:GeoCon}, the collection of which is denoted by $F$.

We also represent the unit sphere in $\mathbb{R}^n$ with $\mathcal{S}^{n-1}$. That is

$$
\mathcal{S}^{n-1} = \left\{x\in \mathbb{R}^n \mid \lVert x\rVert_2= 1 \right\}
$$

We also utilize the concept of spherical codes in this paper. A spherical code with parameters $(n,N,\cos \theta)$ represents a set of $N$ points (say $y_i, i=1,\ldots, N$) on the unit sphere $\mathcal{S}^{n-1} \subset \mathbb{R}^n$ such that the angle between the unit vectors from the origin to any two distinct points is at least $\theta$, i.e.,
$$
\langle y^i, y^j \rangle \le \cos\theta, \;\;\; i\ne j
$$
where $\langle\cdot,\cdot \rangle$ represents the dot product on $\mathbb{R}^n$.

Given these preliminaries, we will now proceed to the construction of the ensemble $F$ along with its properties.

\section{Geometric Construction}\label{sec:GeoCon}
In this section we aim to construct a set of inverse reinforcement learning problems with the intention of applying Fano's inequality to obtain a lower bound for the sample complexity. The geometry of the construction of these problems provides a lower bound of the number of such problems that can be constructed which allows for a Fano's style approach to bounding the sample complexity. First we will describe the geometric construction of the IRL problem sets along with the resulting number of problems constructed. Next we will provide an upper bound for the KL divergence between the densities of the transition probabilities. Finally, we use these results along with Fano's inequality to come up with the sample complexity lower bound.

Consider a set of $n$-state 2-action IRL problems $F = \left\{\mathcal{F}^i\right\}$ where each problem $i$ consists of two possible actions, $a^i_1 = a_1$ and $a^i_2$ with corresponding transition probabilities $P_{a_1}\in [0,1]^{n\times n}$ and $P_{a^i_2}\in [0,1]^{n\times n}$. Let $\gamma \in (0,1)$ be fixed between all the problems. Let $R^i\in \mathbb{R}^n$ be the reward function for problem $i$ that results in action $a_1$ as the optimal policy. Further, let $P_{a_1}$, the transition probability under the desired optimal action $a_1$, be fixed between the set of problems and be given as follows
$$
P_{a_1} = \begin{bmatrix}
\frac{1}{n} & \frac{1}{n} & \frac{1}{n} & \ldots & \frac{1}{n}\\
\frac{1}{n} & \frac{1}{n} & \frac{1}{n} & \ldots & \frac{1}{n}\\ 
\vdots&\vdots&\vdots&\ddots&\vdots\\
\frac{1}{n} & \frac{1}{n} & \frac{1}{n} & \ldots & \frac{1}{n}\\
\end{bmatrix}
$$

We construct the problem pairs $\left(\mathcal{F}^i, R^i\right)$ such that the following two relations, which follow from the Bellman Optimality condition (Eq \ref{eq:BellOpt}) in matrix form, hold
$$
\left(P_{a_1}-P_{a^i_2}\right)\left(I -\gamma P_{a_1}\right)^{-1}R^i \succeq \textbf{0}
$$
and
$$
\left(P_{a_1}-P_{a^i_2}\right)\left(I -\gamma P_{a_1}\right)^{-1}R^j \nsucceq \textbf{0}\;\;\; i\ne j
$$
That is, reward $R^i$ results in $a_1$ being the Bellman optimal action only for its corresponding problem set $\mathcal{F}^i$. Here the notation $\succeq$ represents the entrywise $\ge$ relation. The notation $\textbf{0}$ is used to represent a vector of zeros.

Furthermore we want to enforce $\beta$- strict separability in each problem set. That is, for each problem set $\mathcal{F}^i$, we have $||R^i||_1 = 1$ and

\begin{multline*}
\left(P_{a_1}(j)-P_{a^i_2}(j)\right)\left(I -\gamma P_{a_1}\right)^{-1}R^i \ge\beta >0 \;\;\;\\\forall j =1,\ldots,n
\end{multline*}
We  now look at the construction of such sets of problems $\mathcal{F}^i$

Let $P_a(i)$ represent the $i$-th row of $P_a$. We notice that $P_a(i)$ belongs to the following set
$$
G^1_n := \left\{x\mid x\in \mathbb{R}^n, \; \sum_{i=1}^{n}x_i = 1\right\}
$$

Now notice that if $P_a(i)$, $P_{a_1}(j) \in G^1_n$ then $P_{a_1}(i) - P_{a}(j)$ belongs to the hyperplane $H_n = \left\{x\in \mathbb{R}^n\mid \sum_{i=1}^n x_i = 0\right\}$. We also notice that there exists an invertible rotation matrix $\Pi:\mathbb{R}^n \rightarrow \mathbb{R}^n$ such that the following holds
$$
\Pi(x) = \begin{bmatrix}
y\\0
\end{bmatrix},\;\;\; x\in H_n, \; y\in \mathbb{R}^{n-1}
$$
or alternatively
$$
\text{Proj}_n(\Pi(x)) = \textbf{0},\;\;\; x\in H_n
$$
where $\text{Proj}_n$ represents the projection on to the $e_n$ dimension, i.e., projection on to the orthogonal subspace of the unit vector $e_n$. Here $\{e_i\}_{i= 1,\ldots, n}$ represent the canonical basis of $\mathbb{R}^n$.  

It is also of note that none of the translation, rotation or projection operations described above change the $2$-norm of $P_{a_1}(i) - P_{a}(j)$. Thus, through composition, we can come up with a one-to-one mapping between the points $\left\{P_a(i)\mid ||P_{a_1}(i)- P_{a}(j)||<\varepsilon \right\}$ and the points $\left\{x \mid x\in \mathbb{R}^{n-1},\; ||x||<\varepsilon \right\}$

Now we describe a construction of each problem $\mathcal{F}^i$ by specifying the second transition matrix $P_{a_2^i}$ through $P_{a_1}- P_{a_2^i}$ and the corresponding reward $R^i$ by constructing their equivalents in $\mathbb{R}^{n-1}$.

Consider the unit sphere in $\mathbb{R}^{n-1}$ given by $\mathcal{S}^{n-2}$. Consider the set of points defining a spherical code $(n-1,N,\cos\theta)$ with minimum angle $\theta$ on $\mathcal{S}^{n-2}$ such that $N$ is maximal. From \cite{jenssen2018kissing}  we have the result

\begin{equation}\label{eq:numv}
N \ge (1+o(1))\sqrt{2\pi(n-1)}\frac{\cos \theta}{\sin^{n-2} \theta}
\end{equation}	
 To form the transition probabilities and the corresponding rewards we wish to consider the facets $\mathcal{Y}$ of the simplicial polytope formed by the spherical code. Since the convex polytope formed by the maximal spherical code can be simplicially decomposed \cite{edmonds1970simplicial}, the resulting simplicial decomposition can be used to form a simplicial polytope with the same vertices with the condition that the interiors of the facets, a $n-2$ simplex, are pairwise disjoint. The number of facets of such a simplicial polytope are lower bounded by \cite{barnette1971minimum} 
 \begin{equation}\label{eq:nfacets}
 |\mathcal{Y}| \ge (n-2)N- (n-1)(n-3)
 \end{equation}

  We denote the elements of $\mathcal{Y} $ as $\mathcal{Y}^i$, such that any pair of points $y \in \mathcal{Y}^i$ are neighbors with respect to the spherical code and
$$
|\mathcal{Y}^i \cap \mathcal{Y}^j| \le n-2,\;\;\; i\ne j
$$

To form the set of problems, we consider the pairwise disjoint cones formed by the vertices in each $\mathcal{Y}^i$ and the origin. The corresponding rewards are formed from the centroids of each $\mathcal{Y}^i$. The resulting geometry from the disjoint interiors of the $\mathcal{Y}^i$ ensures that each reward function only results in $a_1$ as the optimal action for the corresponding problem.

 For every $\mathcal{Y}^i$ we denote the centroid of $\mathcal{Y}^i$ as follows

\begin{equation}\label{eq:centroid}
\bar{y}^i = \frac{1}{n-1}\sum_{y \in \mathcal{Y}^i} y
\end{equation}

We also consider the following hyperplanes in $\mathbb{R}^{n-1}$ passing through the elements of the leave-one-out set of $\mathcal{Y}^i$ and the origin as defined by the corresponding normal vectors $p^i_j$. We further impose the constraint that the norm of each $p^i_j$ is constant across all $i$ and $j$. Formally, $p_j^i\in \mathbb{R}^{n-1}$  is defined by the following conditions:
$$
p^{iT}_j y^i_k = 0, \;\;\; j\ne k, \;\;\;1\le k \le 
n
$$
$$
p_j^{iT}  \bar{y}^i >0
$$
$$
\lVert p^i_j\rVert_2 = \varepsilon
$$

Notice that $p^i_j$ is an element of the null space of 
$$
y^i := \begin{bmatrix}
y^{i}_1&
\ldots&
y^{i}_{j-1}&
y^{i}_{j+1}&
\ldots&
y^{i}_n&
\textbf{0}&
\end{bmatrix}^T
$$

where the exponent $T$ represents transpose, such that $\bar{y}^i$  lies in the interior of the cone formed by the hyperplanes. We also notice that as a result of the construction of the set $\mathcal{Y}^i$ and the hyperplanes defined by $p^i_j$
$$
\begin{bmatrix}
p^{iT}_1\\
\vdots\\
p^{iT}_{n-1}
\end{bmatrix} \bar{y}^k \succeq \textbf{0} \iff i=k
$$
If $\hat{y}^i$ is the unit vector in direction of $\bar{y}^i$, then it follows that

$$
\begin{bmatrix}
p^{iT}_1\\
\vdots\\
p^{iT}_{n-1}
\end{bmatrix} \hat{y}^k \succeq \textbf{0} \iff i=k
$$
and also since the centroid  $\bar{y}^i$ of points on the sphere lies in the interior of the sphere since the sphere is a convex shape, while  $\hat{y}^i$ lies on the surface of the sphere. An example of such a spherical code formation along with visualization of the hyperplanes and centroid involved for a single facet is provided in Figure \ref{fig:sph_code}. 

\begin{figure*}[!tb]
	
	\minipage{0.48\textwidth}
	\centering
	\includegraphics[width=0.7\linewidth]{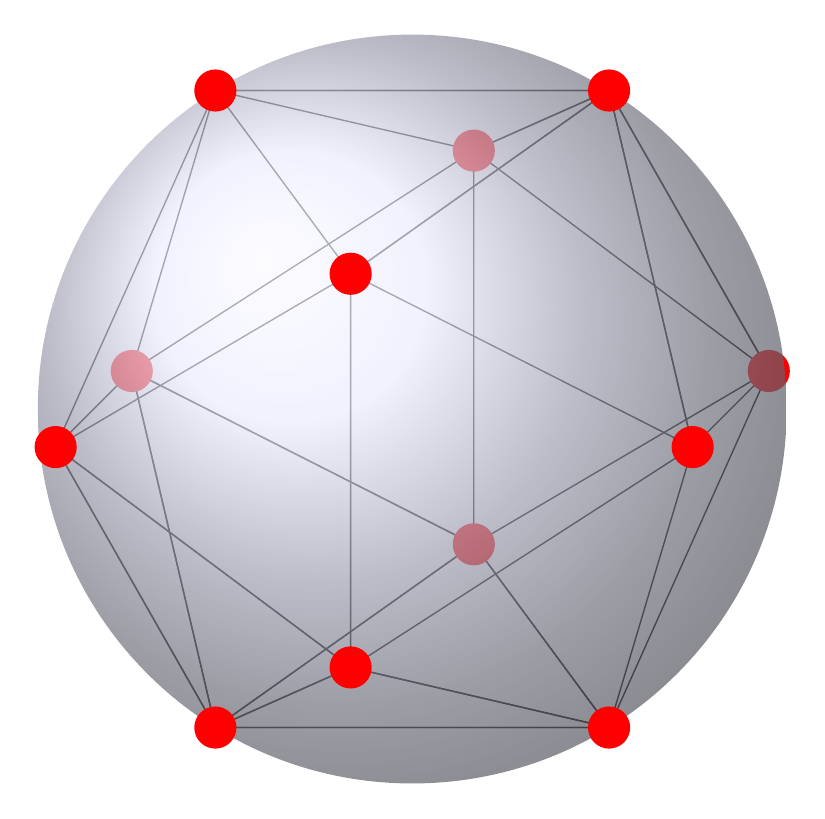}
	
	\endminipage\hfill
	\minipage{0.48\textwidth}
	\centering
	\includegraphics[width=0.7\linewidth]{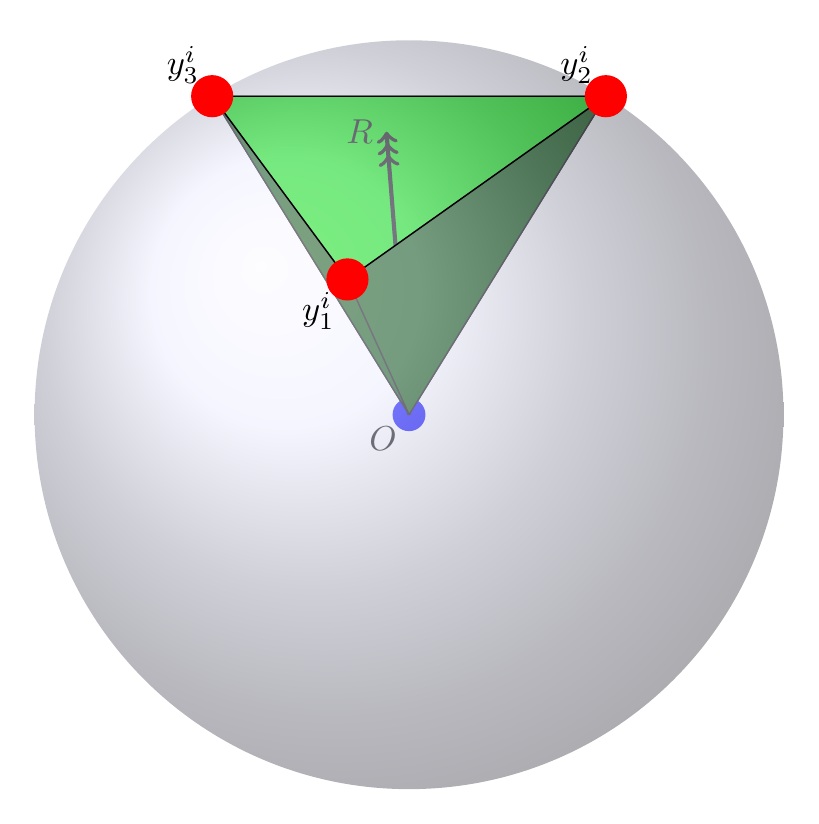}
	
	\endminipage
	\caption{\textbf{Left:} An example graphical visualization a spherical code (kissing number arrangement) on the sphere $\mathcal{S}^2$. The arrangement forms a regular icosahedron. The vertices (elements of $\mathcal{Y}$) are marked in red and the corresponding ridges of the polytope are outlined in black.   \textbf{Right:} An example graphical visualization of a facet ($\mathcal{Y}^i$) from the configuration on the left used in the formation of a single problem. The hyperplanes formed from the leave-one-out set of $\mathcal{Y}^i$ and the origin are shown as the green planes (with normals $p^i_j$) passing through the origin (blue). The reward direction corresponding to $\hat{y}^i$ is shown with the black arrow labeled $R$ }\label{fig:sph_code}
\end{figure*}

\begin{lemma}\label{Lem:dotprod}
	Let $p^i_j$ and $\bar{y}^i$ be as described above. Let $\theta$ be the minimum angle between any pair of points $y^i_j\in \mathcal{Y}$. Let $\hat{y}^i = \bar{y}^i/||\bar{y}^i||_2 $  be the unit vector along $\bar{y}^i$. Then we have
	$$
	\frac{ p^{iT}_j \hat{y}^i}{||p^i_j||_2 } \ge \frac{2\sin \theta/2}{\sqrt{2(n-2)(1+(n-2) \cos \theta)}}
	$$
\end{lemma}
\begin{proof}
	
	We notice that the distance of $\bar{y}^i$ to each hyperplane is greater than or equal to the minimum radius of the ball inside the $n-2$ simplex formed by $\mathcal{Y}^i$ . We can show this from the inner product of $p^i_j$ and $\bar{y}^i$ as well.
	
	Consider $p^{iT}_j \bar{y}^i$, using the expression for the centroid given in Equation \ref{eq:centroid}, we get
	$$
	p^{iT}_j \bar{y}^i = p^{iT}_j \left[\frac{1}{n-1}\sum_{k = 1 }^{n-1} y^i_k\right]
	$$
	since we know from the construction of $p^i_j$ that $p^{iT}y^i_k = 0,\; j\ne k $. Therefore we have,
	$$
	p^{iT}_j \bar{y}^i =\frac{1}{n-1} p^{iT}_j  y^i_j
	$$
	We also notice that the projection of the ray from the origin to the vertex of the $n-2$ simplex, $ y^i_j$, orthogonal to the hyperplane forming the opposing side of the $n-2$ simplex, the normal vector of which is $p^i_j$, is at least the height of the $n-2$ simplex which is given by $s\sqrt{\frac{n-1}{2(n-2)}}$, where $s\ge 2\sin(\theta/2)$ is the edge length of the simplex $\mathcal{Y}^i$. That is
	$$
	\frac{ p^{iT}_j  y^i_j}{\lVert p^{i}_j\rVert_2} \ge  2\sin(\theta/2) \sqrt{\frac{n-1}{2(n-2)}} = \frac{2 \sin \theta/2}{\sqrt{2(n-2)(n-1)}}
	$$
	$$
	\begin{aligned}
	\implies \frac{ p^{iT}_j  \bar{y}^i}{\lVert p^{i}_j\rVert_2} &\ge \frac{1}{n-1} 2\sin(\theta/2) \sqrt{\frac{n-1}{2(n-2)}}\\
	 &= \frac{2 \sin \theta/2}{\sqrt{2(n-2)(n-1)}}
	\end{aligned}
	$$

	Now we also have from Equation \ref{eq:centroid} and the fact that $y^i_k$ are points on a maximal spherical code with minimum angle $\theta$. As a result $\left\langle y^i_j, y^i_k\right\rangle \le
	\cos \theta$ and thus
	$$
	\begin{aligned}
	||\bar{y}^i||_2^2 &= \left\lVert \frac{1}{n-1}\sum_{k = 1 }^{n-1} y^i_k \right\rVert^2\\
	&=\frac{1}{(n-1)^2}\left[\sum_{k = 1 }^{n-1} \left\lVert y^i_k\right\rVert^2 + 2 \sum_{1\le j< k \le n-1 }\left\langle y^i_j, y^i_k\right\rangle \right]\\
	&\le \frac{1}{(n-1)^2} \left[ n-1 + 2 \frac{(n-1)(n-2)}{2}\cos \theta\right]\\
	&=\frac{1+(n-2) \cos \theta}{n-1}
	\end{aligned}
	$$
	$$
	\begin{aligned}
	\implies\frac{ p^{iT}_j \hat{y}^i}{||p^i_j||_2 } &= \\\frac{ p^{iT}_j \bar{y}^i}{||p^i_j||_2 ||\bar{y}^i||_2 } &\ge \frac{2\sin \theta/2}{\sqrt{2(n-2)(1+(n-2) \cos \theta)}}
	\end{aligned}
	$$
\end{proof}

 We will use the $p^i_j$'s to construct the matrix $P_{a_1}- P_{a^i_2}$ and $\bar{y}^i$ to form the corresponding reward $R^i$ through the following transformations
 
 $$
 P_{a_1}(j) - P_{a_2^i}(j) = \left(\Pi^T \begin{bmatrix}
 p^{i}_j\\0
 \end{bmatrix}\right)^T, \;\;\; 1\le j \le n-1
 $$ 
 $$
 P_{a_1}(n) - P_{a_2^i}(n) = \left(\Pi^T \begin{bmatrix}
 p^{i}_1\\0
 \end{bmatrix}\right)^T 
 $$

 \begin{equation}\label{eq:constr}
 R^i = \frac{\left(I - \gamma P_{a_1}\right)\Pi^T \begin{bmatrix}
 	\hat{y}^i\\0
 	\end{bmatrix}}{\left\lVert\left(I - \gamma P_{a_1}\right)\Pi^T \begin{bmatrix}
 	\hat{y}^i\\0
 	\end{bmatrix}\right\rVert_1}
 \end{equation}

 Since there are only $n-1$ vectors $p^{i}_j$ while  $P_{a_1} - P_{a_2^i}$ is an $n \times n$ matrix, we use $p^{i}_1$ for the final row so that the interior of the transformed $n-2$ simplex in $\mathbb{R}^{n-1}$ remains the same. 
 
  Notice that from the above construction, we have the following conditions being met
 
 $$
 \left(P_{a_1}-P_{a^i_2}\right)\left(I -\gamma P_{a_1}\right)^{-1}R^i \succeq 0
 $$
 and
 $$
 \left(P_{a_1}-P_{a^i_2}\right)\left(I -\gamma P_{a_1}\right)^{-1}R^j \nsucceq 0\;\;\; i\ne j
 $$
  Next we will use the $\beta$-strict separability condition to form a relation between $\beta$, $\varepsilon = \lVert p^i_j\rVert_2$ and the angle $\theta$ that generates the spherical code.
  \section{Analysis of Geometric Construction and Sample Complexity of IRL} \label{sec:samp_comp}
The previous section described the geometric construction of the ensemble of IRL problems. We now analyze the geometric construction described in the previous section in order to find the conditions for $\beta$-strict separability of the generated problems. We also derive bounds on the cardinality of the problem ensemble constructed as well as the KL divergence between the trajectories generated by different problems in the ensemble. Finally we use these results along with Fano's inequality \cite{cover2006elements}, to derive an information-theoretic lower bound for the sample complexity of the IRL problem.
  
  \begin{lemma}\label{lem:theta}
 Consider a IRL problem and reward pair constructed as in Equation  \ref{eq:constr}. If the following relation between $\beta$, $\varepsilon $ and $\theta$ holds
 $$
 \sin^2 \frac{\theta}{2} = \frac{n(n-1)(n-2)\beta^2}{2\varepsilon^2 + 2n(n-2)^2\beta^2}
 $$
 then the problem is $\beta$-strict separable.
  \end{lemma}
 \begin{proof}
 	Note that

 \begin{multline*}
(P_{a_1}(j) - P_{a^i_2}(j)) \left(I - \gamma P_{a_1}\right)^{-1} R^i =\\ \left(\Pi^T \begin{bmatrix}
p^{i}_j\\0
\end{bmatrix}\right)^T \left(I - \gamma P_{a_1}\right)^{-1}  \frac{\left(I - \gamma P_{a_1}\right)\Pi^T \begin{bmatrix}
	\hat{y}^i\\0
	\end{bmatrix}}{\left\lVert\left(I - \gamma P_{a_1}\right)\Pi^T \begin{bmatrix}
	\hat{y}^i\\0
	\end{bmatrix}\right\rVert_1} 
 \end{multline*}
 $$
 \begin{aligned}
 =\frac{ \begin{bmatrix}
 	p^{i}_j\\0
 	\end{bmatrix}^T \begin{bmatrix}
 	\hat{y}^i\\0
 	\end{bmatrix}}{\left\lVert\left(I - \gamma P_{a_1}\right)\Pi^T \begin{bmatrix}
 	\hat{y}^i\\0
 	\end{bmatrix}\right\rVert_1} &\ge \frac{\begin{bmatrix}
 	p^{i}_j\\0
 	\end{bmatrix}^T \begin{bmatrix}
 	\hat{y}^i\\0
 	\end{bmatrix}}{\max_y\left\lVert\left(I - \gamma P_{a_1}\right)\Pi^T \begin{bmatrix}
 	y\\0
 	\end{bmatrix}\right\rVert_1} \\ &\ge \beta,\;\;\;\; ||y||_2 =1
 \end{aligned}
 $$

 Notice that since $\Pi$ is a rotation matrix $||\Pi||_2 = 1$. Additionally from the construction of $ \left(I - \gamma P_{a_1}\right)$, we have $|| \left(I - \gamma P_{a_1}\right)||_2 = 1$. Now consider 
 $$
 \max_{y,\; ||y||_2 =1}\left\lVert\left(I - \gamma P_{a_1}\right)\Pi^T \begin{bmatrix}
 y\\0
 \end{bmatrix}\right\rVert_1
 $$
 The above equation represents the norm of $\left(I - \gamma P_{a_1}\right)\Pi^T$ as an operator from $(\mathbb{R}^n, \lVert\cdot\rVert_2)$ to  $(\mathbb{R}^n, \lVert\cdot\rVert_1)$. By duality, this is the same as the norm of the adjoint $\Pi\left(I - \gamma P_{a_1}\right)^T$ as an operator from $(\mathbb{R}^n, \lVert\cdot\rVert_\infty)$ to  $(\mathbb{R}^n, \lVert\cdot\rVert_2)$. Using this we get
 $$
 \max_{ ||y||_2 =1}\left\lVert\left(I - \gamma P_{a_1}\right)\Pi^T \begin{bmatrix}
 y\\0
 \end{bmatrix}\right\rVert_1
 $$
 
 $$
=\max_{||v||_\infty =1}\left\lVert\Pi\left(I - \gamma P_{a_1}\right)^T v\right\rVert_2 =\sqrt{n}
 $$
 
 This gives us
  $$
 (P_{a_1}(j) - P_{a^i_2}(j)) \left(I - \gamma P_{a_1}\right)^{-1} R^i \ge\frac{p^{iT}_j\hat{y}^i}{\sqrt{n}}=\frac{\varepsilon p^{iT}_j\hat{y}^i}{||p^i_j||_2 \sqrt{n}}
 $$

 Substituting the result of Lemma \ref{Lem:dotprod}, we get
  $$
 \frac{\varepsilon p^{iT}_j\hat{y}^i}{||p^i_j||_2 \sqrt{n} }\ge\frac{2\varepsilon\sin \theta/2}{\sqrt{2n(n-2)(1+(n-2) \cos \theta)}}
 $$
 
 Now substituting 
 $$
 \sin \frac{\theta}{2} = \sqrt{ \frac{n(n-1)(n-2)\beta^2}{2\varepsilon^2 + 2n(n-2)^2\beta^2}}
 $$
 and using the trigonometric formula $\cos \theta = 1- 2 \sin^2 \theta/2$ we get
 
  $$
 (P_{a_1}(j) - P_{a^i_2}(j)) \left(I - \gamma P_{a_1}\right)^{-1} R^i \ge \beta
 $$
 
 Since $\lVert R^i \rVert_1 = 1$ by construction, the problem is $\beta$-strict separable
\end{proof}

 We now proceed to use this construction to find a lower bound for the number of such probability matrix - reward function pairs as well as an upper bound on the KL divergence between the corresponding probability matrices. 

\begin{theorem}\label{thm:facet_bound}
	Given a construction of $\beta$-strict separable IRL problems and reward function pairs in Equation \ref{eq:constr}, where the angle $\theta$ of the spherical code used to generate the problems satisfies Lemma \ref{lem:theta}, the minimum number of such problem-reward pairs for a given $n$, $\varepsilon$ and $\beta$ is 
	\begin{multline*}
	|\mathcal{Y}| \ge (n-2) \Bigg( (1+o(1))\sqrt{2\pi(n-1)}\frac{\varepsilon^2-n(n-2)\beta^2}{\varepsilon^2 + n(n-2)^2\beta^2}\times\\
	\left(\frac{\varepsilon^2 + n(n-2)^2\beta}{\sqrt{n(n-1)(n-2)\left(2\varepsilon^2 +n(n-2)(n-3)\beta^2\right)}}\right)^{n-2}\Bigg)\\- (n-1)(n-3)
	\end{multline*}
\end{theorem}
\begin{proof}
	We start with the result of Lemma \ref{lem:theta}.
	
	$$
	\sin \frac{\theta}{2} = \sqrt{ \frac{n(n-1)(n-2)\beta^2}{2\varepsilon^2 + 2n(n-2)^2\beta^2}}
	$$
	From the trigonometric identities $\cos^2 \frac{\theta}{2}+\sin^2 \frac{\theta}{2} = 1$, $\sin\theta = 2 \sin \frac{\theta}{2}\cos \frac{\theta}{2}$ and $\cos\theta = 1-2\sin^2 \frac{\theta}{2}$ we get
	$$
	\sin \theta = \frac{\beta\sqrt{n(n-1)(n-2)\left(2\varepsilon^2 +n(n-2)(n-3)\beta^2\right)}}{\varepsilon^2 + n(n-2)^2\beta^2}
	$$
	$$
	\cos\theta = \frac{\varepsilon^2-n(n-2)\beta^2}{\varepsilon^2 + n(n-2)^2\beta^2}
	$$
	
	Substituting the above into Equation \ref{eq:numv} and subsequently into Equation \ref{eq:nfacets} we get
	\begin{multline*}
	|\mathcal{Y}| \ge (n-2) \Bigg( (1+o(1))\sqrt{2\pi(n-1)}\frac{\varepsilon^2-n(n-2)\beta^2}{\varepsilon^2 + n(n-2)^2\beta^2}\times\\
	\left(\frac{\varepsilon^2 + n(n-2)^2\beta^2}{\beta\sqrt{n(n-1)(n-2)\left(2\varepsilon^2 +n(n-2)(n-3)\beta^2\right)}}\right)^{n-2}\Bigg)\\- (n-1)(n-3)
	\end{multline*}
	
\end{proof}

It is of note that the bounds of $\varepsilon$ from the $\beta$-strict separability condition as well as the condition that the of the minimum ball contained in the probability simplex we have the following result
\begin{lemma}\label{lem:epslowbnd}
Consider a IRL problem and reward pair constructed as in Equation \ref{eq:constr} from an $(n-1,N,\cos\theta)$ spherical code such that Lemma \ref{lem:theta} holds. Then
\begin{equation}
\label{eq:eps_bnd}
\frac{1}{\sqrt{(n-1)(n)}}\ge\varepsilon \ge \sqrt{n-2} \beta
\end{equation}
and the lower bound corresponds to a $n-1$ simplex.
\end{lemma}
\begin{proof}
	The upper bound 
	$$
	\frac{1}{\sqrt{(n-1)(n)}}\ge\varepsilon
	$$
	is straightforwardly obtained from the condition that $P_{a_2^i}(j)$ is contained in a ball of radius $\varepsilon$ around $P_{a_1}(j)$ which is located at the center of the probability simplex. The bound represents the maximum radius ball that fits within the $n-1$ probability simplex with side length $\sqrt{2}$.
	
	The minimum can be found by noticing that the convex $n-1$ dimensional polytope in $\mathbb{R}^{n-1}$ formed by the maximal spherical code must be at least $n-1$-vertex-connected by Balinski's theorem \cite{balinski1961graph}. Thus the minimum number of vertices (which give the minimum possible $\cos \theta$) must be $n$. This minimum is achieved in the form of an $n-1$ simplex with vertices on $\mathcal{S}^{n-2}$ with $\theta = \cos^{-1}\frac{-1}{n-1}$
	
	From this minimum case we have 
	$$
	\cos\theta = \frac{\varepsilon^2-n(n-2)\beta^2}{\varepsilon^2 + n(n-2)^2\beta^2} \ge -\frac{1}{n-1}
	$$
	rearranging and simplifying gives
	$$
	\implies \varepsilon^2 \ge (n-2)\beta^2
	$$
	The solution to which gives the lower bound of the lemma

\end{proof}

Now we apply the result from  Equation (3) of \cite{borade2008euclidean} to bound the KL divergence $D$ of two columns of the transition probability matrices $P$  and $Q$ from two different problems where the columns lie within a ball of radius $\varepsilon$ around $P_{a1}(i)=\begin{bmatrix}
\frac{1}{n} & \frac{1}{n} & \ldots \frac{1}{n} 
\end{bmatrix}$
\begin{lemma} \label{lem:klcol}
	Let $P(i)$ and $Q(j)$ be the columns $i$ and $j$ of the transition probability matrices of such problem-reward pairs as described above constructed using Equation \ref{eq:constr}. Then 
	$$
	D(P(i)||Q(j))\le  \frac{2\varepsilon^2 n}{1-n\varepsilon}
	$$ 
\end{lemma}

\begin{theorem}\label{thm:kltraj}
	Let $P$ and $Q$ be two transition probability matrices of such problem-reward pairs as described above constructed using Equation \ref{eq:constr}. Consider the $m$-length trajectories drawn from each transition probability and let $p^{(m)}$ and $q^{(m)}$ represent the corresponding probability distributions of the trajectory $m$-tuples. Furthermore, let the trajectories start in each state with equal probability. Then we have
	
	$$
	D(p^{(m)}||q^{(m)}) \le (m-1) \frac{2 \varepsilon^2 n}{1-n\varepsilon}
	$$
	
\end{theorem}
\begin{proof}
	We use an approach based on Theorem 1 of \cite{rached2004kullback}.
	Notice that since the columns of $P$ and $Q$ are within the $\varepsilon$ ball around $\begin{bmatrix}
		\frac{1}{n} & \frac{1}{n} & \ldots \frac{1}{n} 
	\end{bmatrix}$ and $\varepsilon < \frac{1}{n}$, none of the elements of $P$ and $Q$ are 0. Thus we have $P$ is absolutely continuous with respect to $Q$.
	
	Now we have
	
	$$
	D(p^{(m)}||q^{(m)}) = p(I+P+P^2 + \ldots + P^{m-2})V + D(p||q)
	$$
	Where $p = q = \begin{bmatrix}
		\frac{1}{n} & \frac{1}{n} & \ldots \frac{1}{n} 
	\end{bmatrix}$ represent the initial distributions of states for the trajectories and $V$ is given by
	$$
	V = \begin{bmatrix}
	D(P(1)||Q(1)) \\ D(P(2)||Q(2)) \\ \vdots \\ D(P(n)||Q(n))
	\end{bmatrix} \le \mathds{1}\frac{2 \varepsilon^2 n}{1-n\varepsilon} \;\;\; (\text{by Lemma \ref{lem:klcol} })
	$$
	
	Also notice that $D(p||q) = 0$. By construction of $P$ and $Q$, we can write $P$ as 
	$
	P = P_{a_1} + \varepsilon U
	$. Since $P$ is also a transition probability matrix, $U$ will be a matrix whose columns are unit vectors whose components sum to 0. That is
	$
	\mathds{1}^T U = \begin{bmatrix}
	1 & 1& \ldots & 1
	\end{bmatrix}U   = 0^T
	$
	$
	\implies P_{a_1} U = 0 I
	$

	Substituting into the expression for $	D(p^{(m)}||q^{(m)})$, and simplifying, we get

	\begin{multline*}
	D(p^{(m)}||q^{(m)}) = \begin{bmatrix}
\frac{1}{n} & \frac{1}{n}& \ldots & \frac{1}{n}
\end{bmatrix} \bigg(I+P_{a_1} + \varepsilon U +\\\left(P_{a_1} + \varepsilon U\right)^2 + \ldots + \left(P_{a_1} + \varepsilon U\right)^{m-2}\bigg)V
	\end{multline*}
	
	Since every term with $U$ will either be multiplied with $P_{a_1}$ or $\begin{bmatrix}
		\frac{1}{n} & \frac{1}{n}& \ldots & \frac{1}{n}
	\end{bmatrix}$ from the left, all of the terms with $U$ will end up being 0.
	
	$$
	\begin{aligned}
	\implies D(p^{(m)}||q^{(m)}) &= \begin{bmatrix}
	\frac{1}{n} & \frac{1}{n}& \ldots & \frac{1}{n}
	\end{bmatrix} \\ &\qquad\left(I+(m-2)P_{a_1}\right) V\\
	&\hspace{-1.2in}\le  \begin{bmatrix}
	\frac{1}{n} & \frac{1}{n}& \ldots & \frac{1}{n}
	\end{bmatrix} \left(I+(m-2)P_{a_1}\right)\mathds{1} \frac{2 \varepsilon^2 n}{1-n\varepsilon}\\
	&\hspace{-1.2in}= \begin{bmatrix}
	\frac{1}{n} & \frac{1}{n}& \ldots & \frac{1}{n}
	\end{bmatrix} \mathds{1}(1+(m-2)) \frac{2 \varepsilon^2 n}{1-n\varepsilon} \\
	&\hspace{-1.2in}= (m-1)\frac{2 \varepsilon^2 n}{1-n\varepsilon}
	\end{aligned}
	$$
	 
\end{proof}

Notice that in the case of multiple actions, if we assume the actions are chosen randomly, thereby creating an extended $P$ and $Q$ matrix for state transition of dimensions $nk \times nk$, the KL bound as well as the result of Theorem \ref{thm:kltraj} do not change.

We will now use Fano's inequality to form the lower bound.
\begin{theorem}\label{thm:fano1}
	Consider the set of constructed problem reward pairs $F = \left\lbrace\mathcal{F}^i, R^i\right\rbrace$ constructed as described above. Let $\mathcal{F}$ be uniform on $F$. Let $Z$ represent the $m$ sample trajectory generated from $\mathcal{F}$ and let $\hat{\mathcal{F}}$ be an estimator of $\mathcal{F}$ from $Z$. Then for any Markov chain $\mathcal{F}\rightarrow Z\rightarrow \hat{\mathcal{F}}$, we have
	$$
	\mathbb{P}(\hat{\mathcal{F}}\ne\mathcal{F})\ge 1- \frac{(m-1)\frac{2 \varepsilon^2 n}{1-n\varepsilon} + \log 2}{\log \eta}
	$$
	where
	\begin{multline*}
\eta = (n-2) \Bigg( \sqrt{2\pi(n-1)}\frac{\varepsilon^2-n(n-2)\beta^2}{\varepsilon^2 + n(n-2)^2\beta^2}\times\\
	\left(\frac{\varepsilon^2 + n(n-2)^2\beta^2}{\beta\sqrt{n(n-1)(n-2)\left(2\varepsilon^2 +n(n-2)(n-3)\beta^2\right)}}\right)^{n-2}\Bigg)\\- (n-1)(n-3)
	\end{multline*}
	
\end{theorem}
\begin{proof}
	We start with Fano's inequality
	$$
	\mathbb{P}(\hat{\mathcal{F}}\ne\mathcal{F})\ge 1- \frac{I(\mathcal{F};Z) + \log 2}{\log |F|}
	$$
	
	Notice that
	$$
	\begin{aligned}
	I(\mathcal{F};Z) &\le \max_{\mathcal{F},\mathcal{F}'} D\left(P_{Z|\mathcal{F}}(\cdot|\mathcal{F})||P_{Z|\mathcal{F}}(\cdot|\mathcal{F}')\right)\\
	&= \max_{p^{(m)},q^{(m)}}D(p^{(m)}||q^{(m)})\\
	&\le (m-1)\frac{2 \varepsilon^2 n}{1-n\varepsilon} \;\;\; \text{(by Theorem \ref{thm:kltraj})}
	\end{aligned}
	$$
	
	We also know that the number of such problem-reward pairs $|F|$ is  the number of facets $|\mathcal{Y}|$. From Theorem \ref{thm:facet_bound}, we have
	
	\begin{multline*}
	|F| \ge (n-2) \Bigg( \sqrt{2\pi(n-1)}\frac{\varepsilon^2-n(n-2)\beta^2}{\varepsilon^2 + n(n-2)^2\beta^2}\times\\
	\left(\frac{\varepsilon^2 + n(n-2)^2\beta^2}{\beta\sqrt{n(n-1)(n-2)\left(2\varepsilon^2 +n(n-2)(n-3)\beta^2\right)}}\right)^{n-2}\Bigg)\\- (n-1)(n-3)
	\end{multline*}
	
	 Substituting these results in Fano's inequality gives us the result of the Theorem.
\end{proof}

\begin{corollary}\label{cor:fano2}
		Consider the set of constructed problem reward pairs $F = \left\lbrace\mathcal{F}^i, R^i\right\rbrace$ constructed as described above. Let $\mathcal{F}$ be uniform on $F$. Let $n$ be large and consider the case 
		$$
		\frac{1}{\sqrt{2n(n-1)}}=\varepsilon = \sqrt{n-2} \beta
		$$
		Let $Z$ represent the $m$ sample trajectory generated from $\mathcal{F}$ and let $\hat{\mathcal{F}}$ be an estimator of $\mathcal{F}$ from $Z$ with
		$$
		m \le (n-1) (0.5 \log n -\log 2)\left(1-\sqrt{\frac{n}{2(n-1)}}\right)+1
		$$
		 Then for any Markov chain $\mathcal{F}\rightarrow Z\rightarrow \hat{\mathcal{F}}$, we have
	$$
	\mathbb{P}(\hat{\mathcal{F}}\ne\mathcal{F})\ge 0.5
	$$
	
\end{corollary}

This result gives us the lower bound for the sample complexity in the case of large $n$ on the order of $O(n \log n)$.

A similar result can be found for the case where just the lower bound of $\varepsilon \ge \sqrt{n-2}\beta$ is satisfied, which from Lemma \ref{lem:epslowbnd} results in the case of an $n-1$ simplex.

\begin{corollary}\label{cor:fano3}
	Consider the set of constructed problem reward pairs $F = \left\lbrace\mathcal{F}^i, R^i\right\rbrace$ constructed as described above. Let $\mathcal{F}$ be uniform on $F$. Let $n$ be large and consider the case 
	$$
	\frac{1}{\sqrt{2n(n-1)}}\ge\varepsilon = \sqrt{n-2} \beta
	$$
	Let $Z$ represent the $m$ sample trajectory generated from $\mathcal{F}$ and let $\hat{\mathcal{F}}$ be an estimator of $\mathcal{F}$ from $Z$ with
	$$
	m \le  \frac{(0.5 \log n -\log 2)}{2(n-2)n\beta^2}\left(1-n\sqrt{n-2}\beta\right)+1
	$$
	Then for any Markov chain $\mathcal{F}\rightarrow Z\rightarrow \hat{\mathcal{F}}$, we have
	$$
	\mathbb{P}(\hat{\mathcal{F}}\ne\mathcal{F})\ge 0.5
	$$
	
\end{corollary}
 
 This result gives us the lower bound for the sample complexity in the case of large $n$ on the order of $O(\frac{ \log n}{n^2\beta^2})$. The results of our experimental validation of this bound for different values of $n$ and various solution methods is described in Section \ref{sec:exp} and can be seen in Figure \ref{fig:exp77} and Figure \ref{fig:exp55}.
 
 \begin{figure*}[!htb]
 	\centering
 	
 	\includegraphics[width=1.1\linewidth]{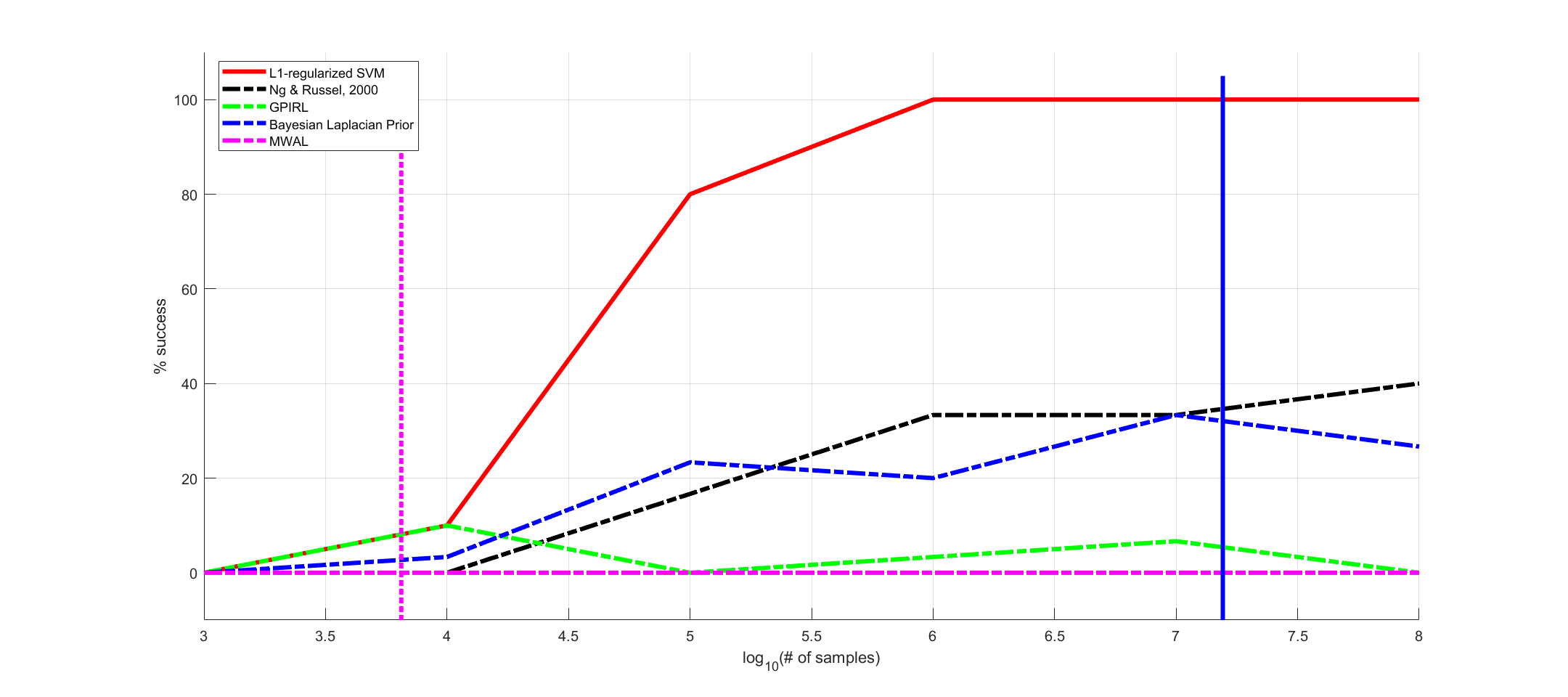}

 	\caption{ Empirical probability of success versus log number of samples for an IRL problem with $n=7$ states, $k=7$ actions, $\gamma  = 0.1$ and $\beta \approx 0.0032$) using the L1-regularized SVM \cite{komanduru2019correctness}, the method of \cite{IRLNg}, Multiplicative Weights for Apprenticeship Learning from \cite{syed2008apprenticeship}, Bayesian IRL with Laplacian prior from \cite{RamaBayes} and Gaussian Process IRL from \cite{levine2011nonlinear}. The vertical magenta line represents the lower bound sample complexity from Corollary \ref{cor:fano3}. The vertical blue line represents the sample complexity upper bound from \cite{komanduru2019correctness}.} \label{fig:exp77}
 \end{figure*}

\begin{figure*}[!htb]
	\centering
	
	\includegraphics[width=1.1\linewidth]{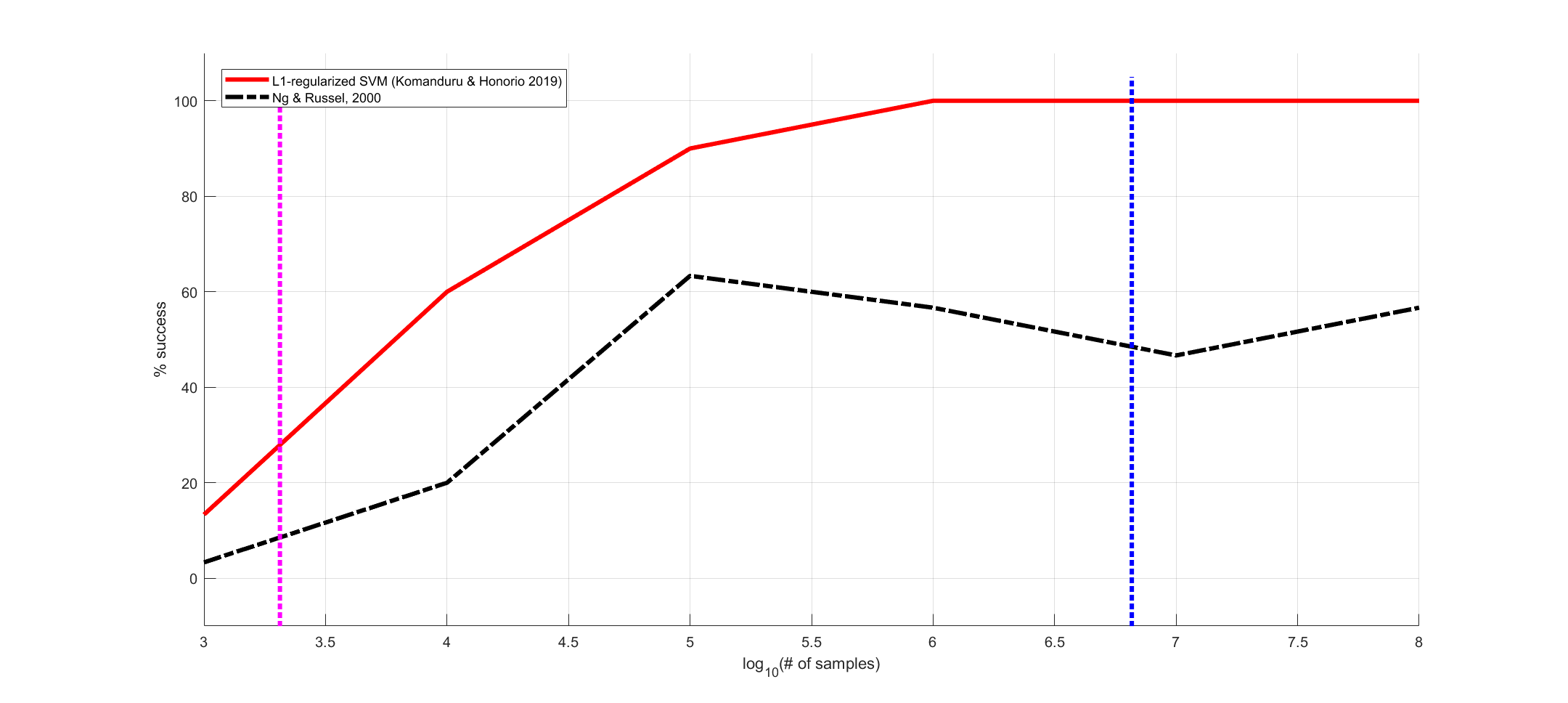}

	\caption{ Empirical probability of success versus log number of samples for an IRL problem with $n=7$ states, $k=7$ actions, $\gamma  = 0.1$ and $\beta \approx 0.0056$) using the L1-regularized SVM \cite{komanduru2019correctness} and the method of \cite{IRLNg}. The vertical magenta line represents the lower bound sample complexity from Corollary \ref{cor:fano3}. The vertical blue line represents the sample complexity upper bound from \cite{komanduru2019correctness}.}\label{fig:exp55}
\end{figure*}

 \section{Simulated Experiments}\label{sec:exp}

We apply our results to the simulated experiment cases performed in \cite{komanduru2019correctness} with a similar metric of percentage of trials where the estimated reward function generates the desired optimal strategy. The choice of this metric reflects the nature of our result: the correct identification of $\mathcal{F}$ from Theorem \ref{thm:fano1} is equivalent to the correct identification of the facet which contains the set of rewards that generate the optimal policy. This is in contrast to other methods that use closeness in the value function generated by the estimated reward function as their metric. We consider the scenario presented in their work (MDP with $n=7$ states, $k=7$ actions, $\gamma  = 0.1$ and $\beta \approx 0.0032$) using the L1-regularized SVM \cite{komanduru2019correctness}, the method of \cite{IRLNg}, Multiplicative Weights for Apprenticeship Learning from \cite{syed2008apprenticeship}, Bayesian IRL with Laplacian prior from \cite{RamaBayes} and Gaussian Process IRL from \cite{levine2011nonlinear}. The results are presented in Figure \ref{fig:exp77}. We also consider another case with  $n=5$ states, $k=5$ actions, $\gamma  = 0.1$ and $\beta \approx 0.0056$ similar to the case presented in \cite{komanduru2019correctness}. The results for this case can be seen in Figure \ref{fig:exp55}.

 In both cases, we observe that the performance of all the methods tested is abysmal ($<50\%$ success) when the number samples is below or close to our predicted lower bound. The performance only starts to improve in various methods when the samples are well above the bound we present. This visibly supports our sample complexity lower bound of $O(\frac{\log n}{n^2 \beta^2})$

\FloatBarrier
% \clearpage
 
\bibliography{AbiBib2019}
\bibliographystyle{plain}
\clearpage
%\newpage
\appendix
\renewcommand\thefigure{\thesection.\arabic{figure}}
\setcounter{figure}{0}
\section{Omitted Proofs}
\subsection{Proof of Lemma \ref{lem:klcol}}
\begin{lemma}
	Let $P(i)$ and $Q(j)$ be the columns $i$ and $j$ of the transition probability matrices of such problem-reward pairs as described above constructed using Equation \ref{eq:constr}. Then 
	$$
	D(P(i)||Q(j))\le  \frac{2\varepsilon^2 n}{1-n\varepsilon}
	$$ 
\end{lemma}
\begin{proof}
	Using Equation (3) of \cite{borade2008euclidean}, if $Q(j) = P(i) + J$, then we have 
	$$
	D(P(i)||Q(j)) \le \frac{1}{2} ||J||^2_{P(i)} 
	$$
	where
	$$
	||J||^2_{P(i)} =  \sum_{j=1}^n \frac{J_j^2}{P_j(i)}
	$$
	Since both $P(i)$ and $Q(j)$ lie in the ball of radius $\varepsilon$ around $\begin{bmatrix}
	\frac{1}{n} & \frac{1}{n} & \ldots \frac{1}{n} 
	\end{bmatrix}$,	we have $\max ||J||_2 = 2\varepsilon$ and $\min P_i(j) = \frac{1}{n}- \varepsilon $ and thus
	
	$$
	||J||^2_P \le  \frac{\max ||J||_2^2}{\min P_i(j)} = \frac{4\varepsilon^2}{1/n - \varepsilon} 
	$$
	
	$$
	\implies D(P(i)||Q(j)) \le \frac{2 \varepsilon^2 n}{1-n\varepsilon}
	$$

\end{proof}

\subsection{Proof of Corollary \ref{cor:fano2}}
\begin{corollary}
	Consider the set of constructed problem reward pairs $F = \left\lbrace\mathcal{F}^i, R^i\right\rbrace$ constructed as described above. Let $\mathcal{F}$ be uniform on $F$. Let $n$ be large and consider the case 
	$$
	\frac{1}{\sqrt{2n(n-1)}}=\varepsilon = \sqrt{n-2} \beta
	$$
	Let $Z$ represent the $m$ sample trajectory generated from $\mathcal{F}$ and let $\hat{\mathcal{F}}$ be an estimator of $\mathcal{F}$ from $Z$ with
	$$
	m \le (n-1) (0.5 \log n -\log 2)\left(1-\sqrt{\frac{n}{2(n-1)}}\right)+1
	$$
	Then for any Markov chain $\mathcal{F}\rightarrow Z\rightarrow \hat{\mathcal{F}}$, we have
	$$
	\mathbb{P}(\hat{\mathcal{F}}\ne\mathcal{F})\ge 0.5
	$$
	
\end{corollary}
\begin{proof}
	From Lemma \ref{lem:epslowbnd}, we know that the case $\varepsilon = \sqrt{n-2} \beta$ corresponds to the spherical code being an $n-1$-simplex which has $n$ facets. Since the number of problems is the number of facets of the convex polytope, substituting $|F| = n$ in the proof of Theorem \ref{thm:fano1} along with $\varepsilon = \frac{1}{\sqrt{2n(n-1)}}$ gives us
	$$
	\begin{aligned}
	&\mathbb{P}(\hat{\mathcal{F}}\ne\mathcal{F})\ge 1- \frac{(m-1)\frac{2 \varepsilon^2 n}{1-n\varepsilon} + \log 2}{\log |F|}\\
	&= 1- \frac{(n-1) (0.5 \log n -\log 2)\frac{1}{n-1} + \log 2}{\log n}\\
	&= 1- 0.5 = 0.5
	\end{aligned}
	$$

\end{proof}

\subsection{Proof of Corollary \ref{cor:fano3}}
\begin{corollary}
	Consider the set of constructed problem reward pairs $F = \left\lbrace\mathcal{F}^i, R^i\right\rbrace$ constructed as described above. Let $\mathcal{F}$ be uniform on $F$. Let $n$ be large and consider the case 
	$$
	\frac{1}{\sqrt{2n(n-1)}}\ge\varepsilon = \sqrt{n-2} \beta
	$$
	Let $Z$ represent the $m$ sample trajectory generated from $\mathcal{F}$ and let $\hat{\mathcal{F}}$ be an estimator of $\mathcal{F}$ from $Z$ with
	$$
	m \le  \frac{(0.5 \log n -\log 2)}{2(n-2)n\beta^2}\left(1-n\sqrt{n-2}\beta\right)+1
	$$
	Then for any Markov chain $\mathcal{F}\rightarrow Z\rightarrow \hat{\mathcal{F}}$, we have
	$$
	\mathbb{P}(\hat{\mathcal{F}}\ne\mathcal{F})\ge 0.5
	$$
	
\end{corollary}
\begin{proof}
	From Lemma \ref{lem:epslowbnd}, we know that the case $\varepsilon = \sqrt{n-2} \beta$ corresponds to the spherical code being an $n-1$-simplex which has $n$ facets. Since the number of problems is the number of facets of the convex polytope, substituting $|F| = n$ in the proof of Theorem \ref{thm:fano1} along with $\varepsilon = \sqrt{n-2}\beta$ gives us
	$$
	\begin{aligned}
	&\mathbb{P}(\hat{\mathcal{F}}\ne\mathcal{F})\ge 1- \frac{(m-1)\frac{2  n(n-2)\beta^2}{1-n\sqrt{n-2}\beta} + \log 2}{\log |F|}\\
	&= 1- \frac{ (0.5 \log n -\log 2) + \log 2}{\log n}\\
	&= 1- 0.5 = 0.5
	\end{aligned}
	$$

\end{proof}

\end{document}